%% file: root_arxiv.tex
\documentclass[letterpaper, 10 pt, conference]{IEEEtran}  

\IEEEoverridecommandlockouts

\usepackage{cite}
\usepackage{amsmath,amssymb,amsfonts}
\allowdisplaybreaks
\usepackage{graphicx}
\usepackage{textcomp}
\usepackage{xcolor}
\usepackage[english]{babel}
\usepackage{amsthm}
\usepackage{bm}
\usepackage{array}
\usepackage{stfloats}
\usepackage{url}
\usepackage{hyperref}
\usepackage{verbatim}
\usepackage{multirow}
\usepackage{booktabs}
\usepackage{algpseudocode}
\usepackage[linesnumbered, ruled, vlined]{algorithm2e}

\newtheorem{theorem}{Theorem}

\newtheorem{lemma}[theorem]{Lemma}

\newtheorem{remark}{Remark}

\title{\LARGE \bf
Unified Branch-and-Bound Search for the\\ Steiner Traveling Salesman Problem on Graphs of Convex Sets
}

\author{Jingtao Tang and Hang Ma%
\thanks{The authors are with the School of Computing Science, Simon Fraser University, Burnaby, BC V5A1S6, Canada. {\tt\footnotesize \{jingtao\_tang, hangma\}@sfu.ca}. 
This work has been submitted to the IEEE for possible publication. Copyright may be transferred without notice.}}

\begin{document}

\maketitle

\input{intro}
\input{problem}
\input{search}
\input{lbg}
\input{exp}
\input{conclusion}

\bibliographystyle{IEEEtran}
\bibliography{ref}

\end{document}

%% file: intro.tex
\begin{abstract}
We formalize the Steiner Traveling Salesman Problem (Steiner-TSP) on Graphs of Convex Sets (GCS), which seeks a minimum-cost closed trajectory through required convex sets while allowing optional transit vertices and revisits.
To explore the resulting infinite solution space, we propose a unified branch-and-bound search over rooted walk prefixes.
Additive lower-bound-graph costs bound committed prefixes, while a cut-separated connected-flow relaxation lower-bounds the residual cost of visiting every remaining target and returning to the root.
Under a uniform positive-cost assumption, best-first traversal terminates after finitely many expansions on every feasible instance without an initial incumbent, whereas depth-first traversal does so once a finite incumbent is available.
For a user-specified factor $\epsilon\geq1$, a global lower bound certifies that either strategy's incumbent cost is at most $\epsilon$ times the global optimum.
We further demonstrate joint sensing-mode, visitation-order, and continuous-trajectory selection for a mobile-manipulator inspection task, including action precedences expressed in linear temporal logic over finite traces (LTL$_f$).
Both traversal strategies find feasible solutions on all benchmark instances within $30\,\mathrm{s}$ with mean certified optimality gaps of $28.1\%$ and $29.7\%$, respectively, whereas two recent baselines succeed on only about half of the instances.
\end{abstract}

\section{Introduction}

Planning a minimum-cost trajectory visiting multiple task regions couples discrete route selection with continuous trajectory optimization.
Graphs of Convex Sets (GCS)~\cite{marcucci2024graphs} encode candidate routes as discrete graph choices and their associated trajectories as continuous variables subject to convex costs and constraints.
Although trajectory optimization is convex for a fixed route, jointly optimizing the discrete route and continuous trajectory is NP-hard in general~\cite{marcucci2024shortest}.
GCS has enabled applications including contact-rich manipulation~\cite{graesdal2024tight,chia2024gcs} and multi-robot motion planning~\cite{tang2025spacetime,zhao2025cb,tang2026search}.
Building on the single-query shortest-path formulation, recent work has considered multi-query shortest paths~\cite{morozov2024multi}, repeated-vertex shortest walks~\cite{morozov2025mixed}, and temporal-logic planning via product GCSs~\cite{kurtz2023temporal,wei2025hierarchical} or augmented GCSs~\cite{you2026framework}.
For multi-region routing, the Traveling Salesman Problem (TSP) on GCS seeks a minimum-cost closed trajectory that visits every GCS vertex while jointly choosing the visitation order and continuous trajectory variables~\cite{marcucci2024graphs,tang2026ghost,luna2026augmented}.
Existing methods address these coupled decisions through a unified Mixed-Integer Convex Program (MICP)~\cite{marcucci2024graphs}, hierarchical tour and path search~\cite{tang2026ghost}, or an augmented GCS with exponentially many target-subset layers~\cite{luna2026augmented}.
In many routing tasks, however, only selected GCS vertices correspond to required task regions.

\begin{figure}[t]
    \centering
    \includegraphics[width=0.98\linewidth]{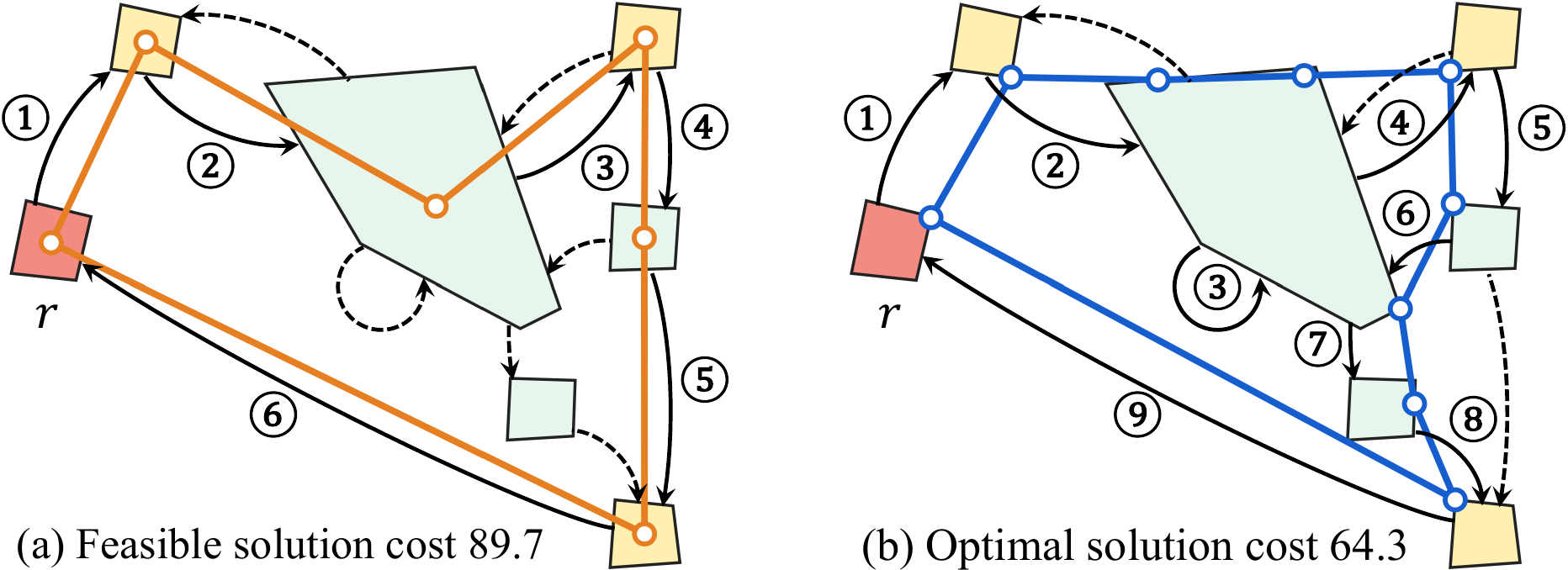}
    \caption{Steiner-TSP on a GCS seeks a minimum-cost closed trajectory from root $r$ (red) visiting all target convex sets (yellow). Black dashed arrows denote directed GCS edges (solid if used). The orange solution (a) is feasible without revisits, whereas the  blue solution (b) exploits vertex revisits and is globally optimal, with circled numbers identifying the underlying GCS edges.}
    \label{fig:steiner-tsp-gcs}
\end{figure}

We study the \emph{Steiner-TSP on GCS}, a required-subset generalization of the TSP on GCS.
Only a designated subset of vertices must be visited, while non-required vertices remain available for transit and vertices may be revisited, consistent with discrete Steiner-TSP semantics~\cite{RODRIGUEZPEREIRA2019615}.
A prior inspection-planning demonstration adapted a high-level restricted-TSP~\cite{hamacher1985k} formulation to require visits to selected GCS vertices~\cite{tang2026ghost}.
Here, we formalize the general problem and study its solution structure.
Compared with finding a shortest path in GCS~\cite{marcucci2024shortest}, the Steiner-TSP on GCS must additionally choose the visitation order of the required vertices, the walks between successive visits, and the return walk to the root.
Depending on the GCS topology, visiting all required vertices and returning to the root may require vertex revisits~\cite{tang2026ghost}; even when unnecessary for connectivity, revisits may reduce trajectory cost~\cite{morozov2025mixed}.
Candidate solutions therefore cannot be restricted a priori to walks without revisits.
Once cyclic subwalks are admitted, the number of repetitions is unbounded, so even a finite GCS can induce infinitely many finite candidate walks.
Efficiently searching this infinite family is therefore a central algorithmic challenge in solving the Steiner-TSP on GCS.

Recent work has explored graph search as an alternative for GCS optimization.
For shortest-path problems, some methods interleave discrete expansion with convex optimization of partial or finite-horizon trajectories~\cite{natarajan2024ixg,chia2024gcs,tang2026search}, whereas another line grows a vertex subset and repeatedly solves restricted GCS relaxations to obtain bounds~\cite{sundar2025bounding}.
Applying graph search to the Steiner-TSP on GCS raises a central design question: when should continuous trajectory optimization be performed?
Optimizing every walk prefix can be expensive.
Deferring full trajectory optimization until a complete target-covering closed walk is found, however, requires efficient lower bounds on the committed-prefix cost and the cost of visiting the remaining targets and returning to the root.

Our unified branch-and-bound search traverses the rooted-walk prefix tree in Fig.~\ref{fig:steiner-tsp-gcs}, maintaining one frontier of live prefixes (Sec.~\ref{sec:unified-search}).
The tree represents every finite rooted walk admitted by the problem, including cycles and vertex revisits.
To order and prune the frontier without optimizing each prefix, we combine an additive committed-prefix cost bound (Sec.~\ref{subsec:cost-to-come}) with a cut-separated connected-flow relaxation lower-bounding the cost of visiting every remaining target and returning to the root (Sec.~\ref{subsec:cost-to-go}).
Their admissible sum reserves full trajectory optimization for complete target-covering closed walks (Sec.~\ref{subsec:search-procedure}).
Under the uniform positive-cost assumption, best-first traversal terminates after finitely many expansions on every feasible instance without an initial incumbent; depth-first terminates after finitely many additional expansions once a finite incumbent exists.
When either strategy terminates with a finite incumbent, the global lower bound certifies that its cost is at most $\epsilon$ times the global optimum for a user-specified factor $\epsilon\geq1$ (Sec.~\ref{subsec:theoretical-guarantees}).
Both variants return finite incumbents on all 180 benchmark instances within the $30\,\mathrm{s}$ budget (Sec.~\ref{sec:numerical-results}).
A mobile-manipulator case study demonstrates joint sensing-mode, visitation-order, GCS-walk, and continuous-trajectory selection with LTL$_f$ action precedences (Sec.~\ref{sec:case-study}).

%% file: problem.tex
\section{Steiner-TSP on GCS}
\label{sec:problem}

A GCS is a finite directed graph $G=(V,E,\mathcal{X})$.
Each vertex $v\in V$ is associated with a nonempty compact convex set $\mathcal{X}_v$ and a positive convex cost $c_v:\mathcal{X}_v\to\mathbb{R}_{> 0}$. 
Each edge $e=(u,v)\in E$ represents an allowed transition, with a closed convex feasible set $\mathcal{X}_e\subseteq\mathcal{X}_u\times\mathcal{X}_v$ and a nonnegative convex cost $c_e:\mathcal{X}_e\to\mathbb{R}_{\geq 0}$.
We assume a uniform per-vertex minimum cost $\eta>0$ such that $c_v(\mathbf{x})\geq\eta$ for every $v\in V$ and $\mathbf{x}\in\mathcal{X}_v$.
Notably, we allow a self-loop $(v,v)\in E$ for each $v\in V$.
A \textit{walk} $\pi=(v_1,\ldots,v_k)$ in $G$ is a finite sequence of vertices such that $(v_{i},v_{i+1})\in E$ for all $i=1,\ldots,k-1$, where vertices may repeat to support more general applications or solutions with smaller costs~\cite{morozov2025mixed} (see also Fig.~\ref{fig:steiner-tsp-gcs}).
Unlike a \textit{path} that disallows vertex revisits, a walk may be necessary to form a closed route when $G$ is incomplete.
A \textit{trajectory} $\tau=(\mathbf{x}_1,\ldots,\mathbf{x}_k)$ is \textit{conditioned} on the walk $\pi$ if $\mathbf{x}_i\in\mathcal{X}_{v_i}$ for $i=1,\ldots,k$ and $(\mathbf{x}_{i},\mathbf{x}_{i+1})\in\mathcal{X}_{(v_{i},v_{i+1})}$ for $i=1,\ldots,k-1$.
Each occurrence of a repeated vertex has its own continuous trajectory variable. For a trajectory conditioned on a closed walk, however, we require $\mathbf{x}_k=\mathbf{x}_1$ and call the trajectory \textit{closed}. By convention, the terminal occurrence incurs no additional vertex cost.
The trajectory cost is therefore
\begin{equation}
    c(\tau):=\sum_{i=1}^{k-1}\left[
    c_{v_{i}}(\mathbf{x}_{i})+c_{(v_{i},v_{i+1})}(\mathbf{x}_{i},\mathbf{x}_{i+1})\right].\label{eq:trajectory-cost}
\end{equation}

Given a target set $V_t\subseteq V$ containing at least two vertices and a designated root $r\in V_t$, the \textit{Steiner-TSP on GCS} is
\begin{subequations}
\begin{align}
    \min_{\pi,\tau}\quad & c(\tau) \\
    \textbf{s.t.}\quad
    &(v_{i-1},v_{i})\in E, && i=2,\ldots,k, \\
    &v_1=v_k=r \text{ and } \mathbf{x}_1=\mathbf{x}_k,\label{eq:steiner-tsp-on-gcs:root}\\
    &V_t\subseteq\{v_1,\ldots,v_{k-1}\},\label{eq:steiner-tsp-on-gcs:cover}\\
    &\tau\text{ is conditioned on }\pi.
\end{align}
\label{eq:steiner-tsp-on-gcs}
\end{subequations}
A feasible solution consists of a closed trajectory conditioned on a closed walk that visits every target vertex at least once.
This formulation generalizes the TSP on GCS~\cite{marcucci2024graphs,tang2026ghost}, which is recovered when $V_t=V$.
Its required-subset semantics also coincide with the labeled-target setting of \cite{luna2026augmented}, and the Steiner terminology emphasizes that non-target GCS vertices remain available for transit without being required.

The Steiner-TSP on GCS couples the discrete walk $\pi$ with the continuous trajectory $\tau$.
A \textit{convex restriction} fixes $\pi$ in Eqn.~\eqref{eq:steiner-tsp-on-gcs} and minimizes $c(\tau)$ over the continuous variables consistent with that walk.
MICP~\cite{marcucci2024graphs} optimizes $\pi$ and $\tau$ jointly within its formulation, whereas GHOST uses a two-level search~\cite{tang2026ghost}: a high-level restricted-TSP search explores target tours, while a low-level search constructs their realizing GCS walks and evaluates complete target-covering closed walks by convex restriction.
Both baselines allow vertices to recur across different inter-target segments. MICP requires each segment to be vertex-simple, while GHOST prevents vertex revisits while searching toward the next target and advances in the selected tour once that target becomes adjacent.

%% file: search.tex
\section{Unified Branch-and-Bound Search}
\label{sec:unified-search}

The Steiner-TSP on GCS has an intrinsic solution space with coupled discrete and continuous components.
Its discrete component consists of finite rooted GCS walks, while each target-covering return to the root indexes a walk-conditioned convex restriction over the continuous trajectory variables; this restriction may be infeasible.
As vertices may be revisited, this family of walks is generally infinite, and unrolling it yields a finitely branching prefix tree determined by the problem.
Our unified branch-and-bound search systematically explores this tree by maintaining a single set of live walk prefixes and using the bounds developed in Sec.~\ref{sec:lbg-relaxation} to prioritize and prune them.

\subsection{Search Space}\label{subsec:search_space}

Formally, the discrete tree contains a search node $n$ for each prefix walk $n.\pi=(v_1=r,v_2,\ldots,v_k)$.
This full-prefix identity is essential because two prefixes reaching the same frontier vertex $v_k$ can induce different feasible boundary conditions and optimal conditioned trajectories even when followed by the same suffix~\cite{chia2024gcs,tang2026search}.
Let $V(n.\pi):=\{v_1,\ldots,v_k\}$ denote the set of vertices visited by $n.\pi$.
For any successor $w$ of $v_k$, let $n.\pi\oplus w:=(v_1,\ldots,v_k,w)$ denote the prefix walk obtained by appending $w$.
A node $n$ is a solution node if $V_t\subseteq V(n.\pi)$ and $v_k=r$.
This is only a discrete condition: the indexed convex restriction may be infeasible; when feasible, its optimal value is the minimum cost among trajectories conditioned on $n.\pi$.
Every node $n$, including a solution node, has a child $n_c$ with $n_c.\pi=n.\pi\oplus w$ for each edge $(v_k,w)\in E$.
Consequently, the tree contains every finite rooted GCS walk, including walks that continue after an earlier target-covering return to $r$; together with the convex restrictions indexed by its solution nodes, it represents every feasible pair $(\pi,\tau)$ of Eqn.~\eqref{eq:steiner-tsp-on-gcs}.

The tree is defined independently of how it is explored.
Alg.~\ref{alg:best-first-bnb} materializes it lazily by maintaining a set of live search nodes, called the \emph{Frontier}.
A traversal strategy selects which live node to explore next.
Each node $n$ additionally stores a cost-to-come lower bound $\hat g(n)$ and a cost-to-go heuristic $h(n)$.
We call $h$ \textit{admissible} with respect to $\hat g$ if, for every feasible solution $(\pi',\tau)$ whose walk extends $n.\pi$, $h(n)\leq c(\tau)-\hat g(n)$.
Sec.~\ref{subsec:search-procedure} combines these quantities into the node lower bound used for Frontier ordering and pruning.

\subsection{Search Procedure}\label{subsec:search-procedure}

We first define the node lower bound and traversal strategy used by the proposed Alg.~\ref{alg:best-first-bnb}.
For a search node $n$ and its child $n_c$ induced by a successor $w$, the shared relaxation developed in Sec.~\ref{sec:lbg-relaxation} supplies an incremental lower-bound cost $\ell(n,w)$, yielding $\hat g(n_c)=\hat g(n)+\ell(n,w)$.
Since each extension commits one vertex cost, the uniform positive-cost assumption in Sec.~\ref{sec:problem} gives $\ell(n,w)\geq\eta$.
Together with the admissible cost-to-go heuristic, this defines the node lower bound
\begin{equation}
    f(n):=\hat g(n)+h(n),
    \label{eq:search-node-lower-bound}
\end{equation}
which lower-bounds the cost of every feasible solution whose walk extends $n.\pi$.
The traversal strategy controls only which live node is selected.
We describe two canonical choices: best-first maintains a min-priority queue ordered by $f(n)$, whereas depth-first uses a last-in-first-out stack and locally orders siblings by $f(n)$.
All traversal strategies prune any node satisfying $f(n)\geq\overline c$ and, once $\overline c<+\infty$, use the global stopping condition $\overline c\leq\epsilon\min_{n\in\textnormal{Frontier}}f(n)$ for $\epsilon\geq1$.

Alg.~\ref{alg:best-first-bnb} accepts an optional initial incumbent to prune in the initial search stage.
Our implementation constructs one by solving the convex restriction for a nearest-neighbor closed walk that repeatedly follows a shortest graph path to the closest unvisited target and finally returns to $r$.
If its restriction is infeasible, no initial incumbent is supplied and $\overline c=+\infty$.

Alg.~\ref{alg:best-first-bnb} initializes Frontier with one two-vertex walk $(r,v)$ for every outgoing edge of $r$ (lines~\ref{line:search-init-frontier}--\ref{line:search-init-push}).
Each iteration checks the global stopping condition, selects a node according to the traversal strategy, and discards it if its bound cannot improve the incumbent (lines~\ref{line:search-loop}--\ref{line:search-prune-continue}).
An unpruned solution node is evaluated by convex restriction and then, like every unpruned node, expanded along all outgoing edges; only children whose cost-to-come and node bounds are below $\overline c$ are inserted (lines~\ref{line:search-solution-test}--\ref{line:search-child-push}).
If Frontier is exhausted or the stopping condition holds, it returns the incumbent if one exists (line~\ref{line:search-return}).

\begin{algorithm}[t]
\DontPrintSemicolon
\linespread{0.8}\selectfont
\caption{Unified Branch-and-Bound Search}
\label{alg:best-first-bnb}
\SetKwInput{KwInput}{Input}
\SetKwInput{KwParam}{Param}
\SetKwInput{KwOutput}{Output}
\KwInput{GCS $G=(V,E,\mathcal{X})$, targets $V_t$, root $r$, optional initial incumbent $\tau^*$ with cost $\overline c$}
\KwParam{traversal strategy, suboptimality factor $\epsilon\geq1$}
\KwOutput{updated incumbent $\tau^*$ and its cost $\overline c$}
$\textnormal{Frontier}\gets\emptyset$\;\label{line:search-init-frontier}
\ForEach{$(r,v)\in E$}{
    $n\gets\textsc{Node}((r,v),\hat g=0)$\;
    \If{$f(n)<\overline c$}{\label{line:search-init-test}$\textnormal{Frontier}.\operatorname{push}(n)$\;\label{line:search-init-push}}
}
\While{$\textnormal{Frontier}\neq\emptyset$}{\label{line:search-loop}
    \If{$\overline c<+\infty$ and $\overline c\leq\epsilon\min_{n\in\textnormal{Frontier}}f(n)$}{\label{line:search-stop-test}
        \textbf{break}\Comment{$\epsilon$-optimal incumbent found}\label{line:search-stop-break}
    }
    $n\gets\textnormal{Frontier}.\operatorname{pop}()$\;\label{line:search-pop}
    \If{$f(n)\geq\overline c$}{\label{line:search-prune-test}\textbf{continue}\label{line:search-prune-continue}}
    \If{$v_k=r$ and $V_t\subseteq V(n.\pi)$}{\label{line:search-solution-test}
        $\tau\gets\textsc{ConvexRestriction}(G,n.\pi)$\;\label{line:search-solution-evaluate}
        \If{$\tau$ is feasible and $c(\tau)<\overline c$}{
            $(\tau^*,\overline c)\gets(\tau,c(\tau))$\;\label{line:search-solution-update}
        }
    }
    \ForEach{$(v_k,w)\in E$}{\label{line:search-expand}
        \If{$\hat g(n)+\ell(n,w)<\overline c$}{
            $n_c\gets\textsc{Node}(n.\pi\oplus w,\hat g=\hat g(n)+\ell(n,w))$\;
            \If{$f(n_c)<\overline c$}{\label{line:search-child-test}$\textnormal{Frontier}.\operatorname{push}(n_c)$\;\label{line:search-child-push}}
        }
    }
}
\Return $(\tau^*,\overline c)$\;\label{line:search-return}
\end{algorithm}

\subsection{Theoretical Properties}\label{subsec:theoretical-guarantees}
We establish finite termination for best-first traversal (Theorem~\ref{thm:finite-search}) and then show that it also holds for depth-first traversal once a finite incumbent becomes available (Remark~\ref{rem:dfs-finite}).

\begin{theorem}[Finite Best-First Search]
\label{thm:finite-search}
Alg.~\ref{alg:best-first-bnb} with best-first traversal terminates after finitely many node expansions for a feasible instance.
\end{theorem}
\begin{proof}
Let $d(n):=|n.\pi|-2$ be the number of extensions after an initial two-vertex node.
Suppose first that initialization leaves $\overline c=+\infty$, and choose any feasible solution $(\bar\pi,\bar\tau)$ with cost $C:=c(\bar\tau)$.
Every node whose prefix lies on $\bar\pi$ satisfies $f(n)\leq C$ by the lower-bound property of $f$.
Before an incumbent is found, $\overline c=+\infty$, so none of these prefixes is pruned.
The proposed relaxation has nonnegative costs, so $h(n)\geq0$.
Since every extension increases $\hat g$ by at least $\eta$, any node with $f(n)\leq C$ satisfies $d(n)\eta\leq\hat g(n)\leq f(n)\leq C$.
Since $G$ is finite, only finitely many nodes have such bounded depth.
Because best-first traversal always selects a minimum-key live node, induction over the finitely many prefixes of $\bar\pi$ shows that it reaches and evaluates the solution node indexed by $\bar\pi$ after finitely many expansions, unless another branch supplies a finite incumbent earlier.

Let $\overline c_0<+\infty$ be the first finite incumbent cost, whether supplied initially or found during search.
Only finitely many nodes have been generated when it becomes available.
Every node inserted afterward satisfies $\hat g(n)<\overline c\leq\overline c_0$ and hence $d(n)<\overline c_0/\eta$.
Finite branching thus allows only finitely many additional nodes to be generated for finite termination.
\end{proof}

\begin{remark}[Depth-First Traversal]~\label{rem:dfs-finite}
The finite-depth argument in the proof of Theorem~\ref{thm:finite-search} is independent of the traversal strategy.
Thus, depth-first traversal also terminates after finitely many additional expansions once a finite incumbent is found during search or supplied initially.
Before then, pure last-in-first-out traversal may starve a live sibling on the infinite tree.
\end{remark}

Let $c^\star$ denote the optimal value of Eqn.~\eqref{eq:steiner-tsp-on-gcs}.
The following Theorem~\ref{thm:epsilon-optimality} show that the same lower-bound property yields an $\epsilon$-optimality certificate for any traversal strategy.

\begin{theorem}[$\epsilon$-optimality]
\label{thm:epsilon-optimality}
For any traversal strategy, if Alg.~\ref{alg:best-first-bnb} terminates with a finite incumbent of cost $\overline c$ by satisfying the global stopping condition or exhausting Frontier, then $\overline c\leq\epsilon c^\star$.
\end{theorem}
\begin{proof}
At termination, let
\begin{equation*}
    L:=\min\left\{\overline c,\min_{n\in\textnormal{Frontier}}f(n)\right\},
\end{equation*}
where the Frontier minimum is omitted if Frontier is empty.
Consider any feasible solution $(\pi,\tau)$ of cost $C$.
Exhaustive successor generation preserves the branch corresponding to $\pi$ until its complete walk is evaluated, one of its prefixes remains in Frontier, or one of its prefixes is discarded by an incumbent-bound test.
In the first case, the walk-conditioned convex restriction gives $\overline c\leq C$; in the second, $\min_{n\in\textnormal{Frontier}}f(n)\leq C$; and in the third, the lower bound responsible for discarding the prefix gives $\overline c\leq C$.
Therefore $L\leq C$ for every feasible solution, and taking the infimum over all feasible solutions gives $L\leq c^\star$.
If the global stopping condition holds, then $\overline c\leq\epsilon L\leq\epsilon c^\star$.
If Frontier is exhausted, $L=\overline c\leq c^\star$.
Since the incumbent is feasible, $c^\star\leq\overline c$ and thus $\overline c=c^\star$.
\end{proof}

%% file: lbg.tex
\section{Lower-Bound Graph for Search Bounds}
\label{sec:lbg-relaxation}

Efficient branch-and-bound requires inexpensive cost-to-come and cost-to-go bounds at every prefix, but future extensions can change the continuous trajectory along a prefix and hence its realized cost~\cite{chia2024gcs,tang2026search}.
Prefix-conditioned convex restrictions preserve continuous consistency~\cite{natarajan2024ixg,chia2024gcs,morozov2024multi,tang2026search}, while repeated fractional GCS relaxations over expanding vertex subsets offer another online bounding strategy~\cite{sundar2025bounding}; both require continuous optimization during search, and our ablation shows that solving a prefix restriction at every generated node sharply limits search progress under the fixed runtime budget (Table~\ref{tab:results}).
To avoid this repeated online optimization, we precompute local convex-restriction costs for short GCS walks and compose them in a reusable discrete relaxation that enforces feasibility within each piece while relaxing consistency between adjacent pieces.
This representation supplies a cost-to-come lower bound $\hat g(n)$ for the committed prefix and an admissible cost-to-go heuristic $h(n)$ for visiting the remaining targets and returning to the root, without a GCS trajectory solve at every generated node.
For every feasible solution $(\pi',\tau)$ to Eqn.~\eqref{eq:steiner-tsp-on-gcs} whose walk $\pi'$ extends $n.\pi$, these bounds satisfy $\hat g(n)+h(n)\leq c(\tau)$.

\subsection{Lower-Bound Graph (LBG) Relaxation}
\label{subsec:lbg-construction}

The triplet-based LBG is one such representation and has previously been used for GCS search bounds~\cite{natarajan2024ixg,tang2026ghost}; a closely related triplet relaxation is used in~\cite{tang2026search}.
We adopt the LBG as the shared representation underlying both $\hat g$ and $h$.
Given a GCS $G=(V,E,\mathcal{X})$, we construct its LBG $L=(E,T)$ as follows.
Every length-two GCS walk $(u,v,w)$ defines a triplet $t=(u,v,w)\in T$, directed in $L$ from edge $e_1=(u,v)\in E$ to edge $e_2=(v,w)\in E$.
We assign $t$ the optimal value of the local convex program
\begin{align}
\ell_t:=\min_{\mathbf{x}_u,\mathbf{x}_v,\mathbf{x}_w}
\left[c_v(\mathbf{x}_v)+c_{e_2}(\mathbf{x}_v,\mathbf{x}_w)\right],
\label{eq:lbg-triplet-cost}
\end{align}
where $\mathbf{x}_u\in\mathcal{X}_u$, $\mathbf{x}_v\in\mathcal{X}_v$, and $\mathbf{x}_w\in\mathcal{X}_w$ are constrained by $(\mathbf{x}_u,\mathbf{x}_v)\in\mathcal{X}_{e_1}$ and $(\mathbf{x}_v,\mathbf{x}_w)\in\mathcal{X}_{e_2}$.
If the program in Eqn.~\eqref{eq:lbg-triplet-cost} is infeasible, we set $\ell_t=+\infty$.
The objective of Eqn.~\eqref{eq:lbg-triplet-cost} includes exactly the costs of the middle vertex $v$ and its outgoing edge $e_2$, as assigned in Eqn.~\eqref{eq:trajectory-cost}, while the constraints enforce local feasibility with both neighboring edges.
Thus, $\ell_t$ lower-bounds the trajectory cost committed at vertex $v$ by any feasible trajectory whose walk contains $t$.
The uniform positive-cost assumption in Sec.~\ref{sec:problem} gives
$\ell_t\geq\eta$ for every finite-cost triplet $t\in T$.

Every GCS walk $\pi=(v_1,\ldots,v_k)$ therefore induces an LBG walk through the consecutive edges $(v_i,v_{i+1})$.
Because each $\ell_t$ is computed independently, adjacent triplets need not share consistent GCS variables, so the induced LBG-walk cost lower-bounds the cost of any feasible trajectory conditioned on $\pi$.
Consider any feasible solution to the Steiner-TSP on GCS, with trajectory $\tau$ conditioned on the GCS walk $\pi=(v_1=r,v_2,\ldots,v_k=r)$.
Closing the corresponding LBG walk gives
\begin{equation}
    \hat c(\pi)
    :=\sum_{i=1}^{k-2}\ell_{(v_i,v_{i+1},v_{i+2})}
      +\ell_{(v_{k-1},r,v_2)}
    \leq c(\tau),
    \label{eq:lbg-additive-lower-bound}
\end{equation}
where $\hat c(\pi)$ denotes the cost of the closed LBG walk induced by $\pi$.
This lower bound is the closed-walk case of Theorem~1 in~\cite{tang2026ghost}, extending the open-path result in Theorem~3 of~\cite{natarajan2024ixg}.

\subsection{Cost-to-Come Lower Bound}\label{subsec:cost-to-come}

Consider a search node $n$ with $n.\pi=(v_1=r,v_2,\ldots,v_k)$.
We define its cost-to-come lower bound $\hat g(n)$ as the cost of the corresponding fixed LBG walk (Lemma~\ref{lem:cost-to-come-lower-bound}):
\begin{equation}
    \hat g(n)
    :=\sum_{i=1}^{k-2}\ell_{(v_i,v_{i+1},v_{i+2})}.
    \label{eq:prefix-lower-bound}
\end{equation}
For the initial prefix $n.\pi=(r,v_2)$, the sum is empty, so $\hat g(n)=0$.
Appending a successor $w$ extends the LBG walk by $t=(v_{k-1},v_k,w)\in T$ and produces a child $n_c$ satisfying
\begin{equation}
    \hat g(n_c)
    =\hat g(n)+\ell(n,w)=\hat g(n)+\ell_t,
    \label{eq:prefix-lower-bound-update}
\end{equation}
where the incremental lower bound in Sec.~\ref{subsec:search-procedure} is instantiated as $\ell(n,w):=\ell_t$.
Because $\ell_t\geq\eta$, each extension provides the positive
progress required by the proof of Theorem~\ref{thm:finite-search}.

\begin{lemma}
\label{lem:cost-to-come-lower-bound}
For any search node $n$ and feasible solution $(\pi',\tau)$ whose walk extends $n.\pi$, we have $\hat g(n)\leq\hat c(\pi')\leq c(\tau)$.
\end{lemma}
\begin{proof}
The first inequality holds because $\hat g(n)$ is a partial sum of the nonnegative LBG-walk cost $\hat c(\pi')$, and the second one comes from the LBG lower-bound property~\cite{tang2026search,tang2026ghost,natarajan2024ixg}.
\end{proof}

\subsection{Cut-Separated Connected-Flow Cost-to-Go}
\label{subsec:cost-to-go}

We propose an admissible heuristic $h(n)$ that lower-bounds the cost of visiting every unvisited target and closing the walk.
The heuristic is evaluated online at every search node by a cut-separated connected-flow Linear Program (LP).
Consider LBG $L=(E,T)$.
With a slight abuse of notation, $T$ denotes only the finite-cost triplets throughout this subsection.
For a search node $n$ with $n.\pi=(v_1=r,v_2,\ldots,v_k)$, let
\begin{equation*}
    e_{\mathrm{src}}:=(v_{k-1},v_k)\in E \text{\, and\, }
    V_t(n):=V_t\setminus V(n.\pi).
\end{equation*}
For each $v\in V_t(n)$, define its incoming-edge set
\begin{equation*}
    E_v:=\{(u,v)\in E\}.
\end{equation*}

Because $v$ is unvisited by $n.\pi$, $e_{\mathrm{src}}\notin E_v$.
Because each GCS edge is an LBG vertex, a GCS walk reaches $v$ exactly when its induced LBG walk visits a vertex in $E_v$.
For any feasible solution $(\pi',\tau)$ to Eqn.~\eqref{eq:steiner-tsp-on-gcs} with $n.\pi$ as a prefix of $\pi'$, the residual LBG walk induced by $\pi'$ therefore starts at $e_{\mathrm{src}}$, visits a vertex in $E_v$ for every $v\in V_t(n)$, and returns through the initial edge $(r,v_2)$ to close the walk.

To enforce closure of the LBG walk through the final triplet $(u,r,v_2)$, let $e_{\mathrm{cl}}$ denote an artificial sink vertex.
Let $E^\sharp:=E\cup\{e_{\mathrm{cl}}\}$, and let $T^\sharp$ augment $T$ with, for every closing triplet $(u,r,v_2)\in T$, a terminal transition from $(u,r)$ to $e_{\mathrm{cl}}$ with cost $\ell_{(u,r,v_2)}$.
Define the admissible source-side sets by
\begin{equation*}
    \mathcal{S}(n):=
    \bigcup_{v\in V_t(n)}
    \{S\subseteq E\setminus E_v:e_{\mathrm{src}}\in S\}.
\end{equation*}
For each $S\in\mathcal{S}(n)$, define its outgoing finite-triplet cut by
\begin{equation*}
    \delta^+(S):=\{(u,v,w)\in T:(u,v)\in S,\ (v,w)\notin S\}.
\end{equation*}
For a nonnegative flow $\xi$ on $T^\sharp$, define $\operatorname{div}\xi\in\mathbb{R}^{E^\sharp}$ componentwise, where $(\operatorname{div}\xi)_e$ is the total outgoing minus incoming flow at edge $e\in E^\sharp$.
For each $e\in E^\sharp$, let $\mathbf{1}_e\in\{0,1\}^{E^\sharp}$ denote the unit vector at $e$.
Let $\boldsymbol{\ell}^\sharp$ contain the corresponding costs on $T^\sharp$.
We define $h(n)$ as the optimal value of
\begin{subequations}
\label{eq:connected-lbg-cut}
\begin{align}
\min_{\xi\in\mathbb{R}_{\geq0}^{T^\sharp}}\quad
    &(\boldsymbol{\ell}^\sharp)^\top\xi
    \\
\mathbf{s.t.}\quad
    &\operatorname{div}\xi
        =\mathbf{1}_{e_{\mathrm{src}}}
        -\mathbf{1}_{e_{\mathrm{cl}}},
    \label{eq:connected-lbg-cut-balance}\\
    &\sum_{t\in\delta^+(S)}\xi_t\geq1,
    \quad
    \forall S\in\mathcal{S}(n).
    \label{eq:connected-lbg-connectivity-cut}
\end{align}
\end{subequations}
Eqn.~\eqref{eq:connected-lbg-cut-balance} routes one net unit of cost-bearing flow from $e_{\mathrm{src}}$ to $e_{\mathrm{cl}}$ through a valid closing terminal transition, while permitting additional circulations.
Eqn.~\eqref{eq:connected-lbg-connectivity-cut} requires at least one unit across every directed cut separating $e_{\mathrm{src}}$ from each remaining target's incoming-edge set.
Hence, a disconnected target-covering circulation cannot by itself certify target visitation.
When $V_t(n)=\emptyset$, the LP reduces to a minimum-cost flow from $e_{\mathrm{src}}$ to $e_{\mathrm{cl}}$.
Whenever the LP is infeasible, $h(n)=+\infty$.

By the max-flow/min-cut theorem, these cut constraints are equivalent to sending one auxiliary unit of flow from $e_{\mathrm{src}}$ to each $E_v$ under capacities $\xi$.
This construction adapts the multi-commodity-flow formulation for the classic Steiner-TSP~\cite{letchford2013compact} by replacing each required vertex with its incoming-edge set $E_v$.
We solve this exponential cut formulation by row generation using a persistent restricted master.
After each solve, we use $\xi$ as arc capacities, add any connectivity cuts violated by an $e_{\mathrm{src}}$-to-$E_v$ minimum cut, and reoptimize.
Lemma~\ref{lem:connected-lbg-flow-admissibility} establishes that $h$ is an admissible cost-to-go heuristic.
\begin{lemma}[Admissibility]
\label{lem:connected-lbg-flow-admissibility}
For any search node $n$ and feasible solution $(\pi',\tau)$ whose walk extends $n.\pi$, we have
\begin{equation}
    \hat g(n)+h(n)
    \leq \hat c(\pi')
    \leq c(\tau).
    \label{eq:combined-lbg-lower-bound}
\end{equation}
Thus, $h$ is admissible with respect to $\hat g$ by definition, and $f(n)$ in Eqn.~\eqref{eq:search-node-lower-bound} lower-bounds the cost of every feasible solution whose walk extends $n.\pi$.
\end{lemma}
\begin{proof}
The multiplicities of the finite triplets in the residual LBG walk of $\pi'$, together with its closing terminal transition, define an integral flow $\xi$ satisfying Eqn.~\eqref{eq:connected-lbg-cut-balance}.
For each $S\in\mathcal{S}(n)$, choose a remaining target $v$ such that $S\cap E_v=\emptyset$.
The residual walk starts in $S$ and reaches an edge in $E_v$ outside $S$.
It must therefore cross $\delta^+(S)$ at least once, so every connectivity cut is satisfied.
The resulting objective is $\hat c(\pi')-\hat g(n)$.
Optimality gives $h(n)\leq\hat c(\pi')-\hat g(n)$, and the second inequality follows from Eqn.~\eqref{eq:lbg-additive-lower-bound}.
\end{proof}

%% file: exp.tex
\section{Numerical Results}\label{sec:numerical-results}
This section presents our numerical results for Steiner-TSP on GCS solvers.
We implement the proposed search and its components in \textit{Python} and evaluate the resulting solver on a machine with an \textit{Apple}\textsuperscript{\textregistered} M4 processor and 16 GB of memory.
The GCS-related trajectory optimization uses the \textit{Drake}~\cite{drake} library configured with the \textit{Gurobi} solver~\cite{gurobi}.
Gurobi also solves the persistent restricted master for the cost-to-go bound, while directed minimum-cut computations identify violated connectivity rows.
Source code, numerical results, and additional visualizations will be released upon publication.
More detailed visualizations and simulations are available at \url{https://sites.google.com/view/steiner-tsp-gcs}.

\subsection{Instances}\label{subsec:instances}

We evaluate the Steiner-TSP on GCS across the three task domains shown in Fig.~\ref{fig:instances}.
All instances use the same minimum-time trajectory model~\cite{marcucci2023motion}.
Each visit to a GCS vertex is parameterized by degree-5 B\'ezier curves for the $d$-dimensional configuration and scalar time.
Accordingly, each variable $\mathbf{x}_i\in\mathbb{R}^{6(d+1)}$ in Eqn.~\eqref{eq:steiner-tsp-on-gcs} comprises six $d$-dimensional configuration control points and six scalar time control points.
We impose $C^2$ continuity across internal GCS transitions, constrain all configuration control points to the selected GCS region, and enforce componentwise velocity bounds on each segment.
We impose neither acceleration nor torque limits.
All edge costs are zero, and each vertex cost equals its segment duration, with a minimum duration of $\eta=0.1\,\mathrm{s}$.

Each domain contains 60 instances generated from seeds $0,\ldots,11$ and five controlled sizes.
For \emph{rand}, point configurations move through convex hulls of six points sampled within $1.5\times1.5$ squares on a unit-spaced grid, yielding a size-dependent workspace.
Each seed grows a connected cover to undirected intersection-edge thresholds of $10,20,\ldots,50$, with every polygon treated as a target ($V=V_t$).
For \emph{maze}, point configurations move through a random $10\times10\times10$ unit-voxel recursive-division maze~\cite{buck2011recursivedivision} with one-voxel-thick walls.
We partition its free voxels into a minimum-cardinality set of disjoint axis-aligned cuboids, a 3-D analogue of the minimum-brick decomposition in~\cite{lu2023tmstc}, where cuboids sharing a face define GCS adjacencies.
We select $10,20,\ldots,50$ cuboids distributed across each maze GCS as targets.
For \emph{iiwa}, we use a precomputed collision-free 15-region joint-space GCS for a fixed-base 7-DoF KUKA LBR iiwa 14 in a shelves-and-bins scene, with IRIS-NP regions~\cite{petersen2023growing} bounded by the robot's joint-position limits.
Each instance adds $3,6,9,12$, or $15$ exclusive target boxes with halfwidth at most $0.01\,\mathrm{rad}$, each adjacent only to its owner region.
In all three domains, we build two oppositely directed edges if the configuration convex sets of a GCS vertex pair intersect.
Table~\ref{tab:instance-stats} summarizes the resulting graphs, with Degree computed using undirected adjacency.

\begin{figure}[t]
    \centering
    \includegraphics[width=\linewidth]{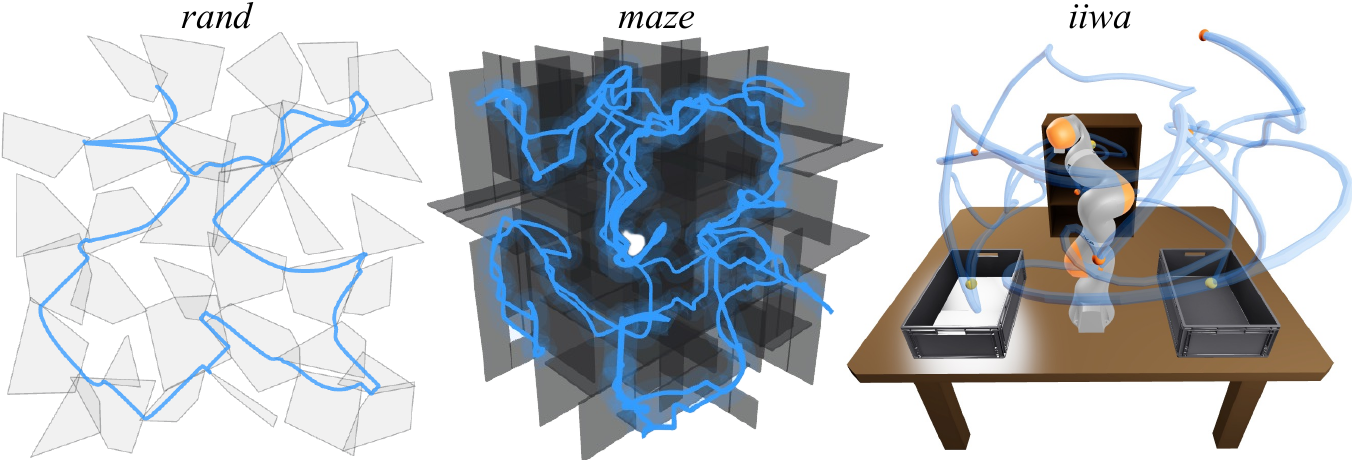}
        \caption{Steiner-TSP on GCS solution trajectories (blue) in \emph{rand} with a 2-D GCS of gray polygons, \emph{maze} with a 3-D GCS exactly partitioning the free voxels outside black walls, and \emph{iiwa} with a 7-D GCS comprising IRIS-NP~\cite{petersen2023growing} regions seeded by representative collision-free configurations.}
    \label{fig:instances}
\end{figure}

\subsection{Solvers}

We compare six solver variants: two traversal strategies for the proposed search, two bound ablations, and two existing solvers.
We instantiate the proposed search with the best-first (BF) and depth-first (DF) traversal strategies, denoted \textbf{Ours (BF)} and \textbf{Ours (DF)}, respectively.
Both use the proposed additive LBG prefix bound $\hat g$ (Sec.~\ref{subsec:cost-to-come}) and cut-separated connected-flow cost-to-go bound $h$ (Sec.~\ref{subsec:cost-to-go}).
The ablation study isolates the two proposed bounds.
\textbf{Alt-G} retains $h$ but replaces $\hat g$ with the optimal committed-prefix cost obtained from a prefix convex restriction at every generated node, whereas \textbf{Alt-H} retains $\hat g$ but replaces $h$ with an LBG heuristic that connects the current prefix to the remaining targets, spans them with a metric-closure MST, and returns to the root.
Both Alt-G and Alt-H use the best-first traversal strategy.
For the performance comparison, we include \textbf{GHOST}~\cite{tang2026ghost} and \textbf{MICP}~\cite{marcucci2024graphs}, as described in Sec.~\ref{sec:problem}.
We provide the active edges of a greedy closed walk as a partial Mixed-Integer Programming (MIP) start for MICP.
The walk repeatedly connects the current vertex to the closest unvisited target by an unweighted shortest path on the input GCS.
All solvers receive a runtime budget $T=30\,\mathrm{s}$.\footnote{For the search-based solvers, LBG precomputation is performed offline and excluded from the runtime budget.
For MICP, model construction and partial MIP-start generation are also excluded.}

\begin{table}[t]
    \centering
    \caption{Steiner-TSP on GCS instance complexity. The $|V|$, $|E|$, Degree, and $|V_t|$ columns report the minimum/mean/maximum. The $d$ column reports the configuration-space dimension, and LBG column reports the mean precomputation time in seconds.}
    \label{tab:instance-stats}
    \setlength{\tabcolsep}{5pt}
    \begin{tabular}{@{}rcccccc@{}}
        \toprule
        \textbf{Domain} & $|V|$ & $|E|$ & Degree & $|V_t|$ & $d$ & LBG \\
        \midrule
        \textit{rand} & 8/23.6/41 & 20/60.6/104 & 1/2.51/7 & 8/23.6/41 & 2 & 0.07s \\
        \textit{maze} & 63/63/63 & 124/125.2/130 & 1/1.99/4 & 10/30/50 & 3 & 0.14s \\
        \textit{iiwa} & 18/24/30 & 104/116.0/128 & 1/4.93/13 & 3/9/15 & 7 & 4.57s \\
        \bottomrule
    \end{tabular}
\end{table}

\subsection{Metrics}
We evaluate terminal quality, anytime performance, and search effort under runtime budget $T$, using as each instance's \emph{reference cost} the lowest incumbent at $T$ across all six solver variants.
\textbf{Cost Regret} is the terminal incumbent's fractional excess over this reference, whereas \textbf{Gap} is the fractional difference between the terminal incumbent and the solver-reported terminal lower bound, capped at $100\%$.
Both are averaged only over runs with a finite incumbent, with a missing finite lower bound assigned a $100\%$ gap; because GHOST and MICP restrict the walk family (Sec.~\ref{sec:problem}), their gaps do not certify the unrestricted problem.
\textbf{Feasible} counts runs with a finite incumbent at $T$.

Following~\cite{berthold2013measuring}, the Time-Normalized Primal Integral (\textbf{NPI}) averages over $[0,T]$ the incumbent's excess over the reference cost, normalized by the current incumbent and capped at one.
Periods without an incumbent contribute one, so lower NPI reflects earlier feasibility and better anytime solution quality.

For search effort, \textbf{Closed Nodes} counts popped and processed frontier nodes, whereas \textbf{Evaluated Walks} counts complete target-covering closed walks passed to a convex restriction.
Closed Nodes is omitted for GHOST and MICP, and Evaluated Walks is omitted for MICP; for the four search-based variants, Evaluated Walks includes the greedy initial walk but excludes Alt-G's prefix restrictions, which are not complete walks.
Both terminal counters may include the final operation begun before $T$ and completed afterward.

\subsection{Result Analysis}
Table~\ref{tab:results} shows that both proposed variants find incumbents on all 180 instances, compared with 99 for GHOST and 98 for MICP, while their lower NPI indicates earlier near-reference solutions.
Ours (BF) performs best across the three quality metrics on \emph{rand} and \emph{iiwa}, whereas Ours (DF) achieves the best gap and NPI on \emph{maze} and ties GHOST in displayed regret.

GHOST and MICP both exhibit complementary bottlenecks.
GHOST remains competitive when successful, with mean regret at most $3.9\%$, but none of its 81 failed runs evaluates a complete walk: failures exhaust $T$ in the initial restricted-TSP solve for \emph{maze}, in the low-level GCS-walk search for \emph{iiwa}, and at either level for larger \emph{rand} instances.
MICP finds incumbents on all \emph{maze} instances but none on \emph{iiwa}, and all its reported gaps reach the $100\%$ cap.
This contrast plausibly reflects formulation difficulty rather than raw GCS size: MICP creates one GCS copy per target and initializes only selected edge indicators, while a representative 3-target \emph{iiwa} model has 4.28 million nonzeros, $5.4$ times that of a 10-target \emph{maze} model.
This is consistent with easier initialization on nearly tree-structured six-facet maze regions than on \emph{rand}, where all vertices are targets, or on 7-D \emph{iiwa} regions with 48--237 facets.

The ablations confirm $\hat g$ and $h$ are complementary.
Alt-H is inexpensive but loose because its pairwise LBG connections need not form one ordered residual walk that visits every remaining target and returns to the root, leaving many prefixes competitive and closing 314k--871k nodes while evaluating only 1--56 complete walks.
Alt-G instead solves a prefix convex restriction for every generated child before pruning; these solves are not captured by either effort counter and limit the search to 43--736 closed nodes without evaluating a complete walk beyond the greedy initialization.
Combining the inexpensive additive $\hat g$ with the connectivity-aware $h$ avoids both bottlenecks and focuses the frontier on complete walks.

The two traversal strategies also have complementary strengths.
Ours (BF) globally prioritizes prefixes by $\hat g+h$ and achieves lower NPI on \emph{rand} and \emph{iiwa}, whereas Ours (DF) uses the same bound for pruning and sibling ordering and performs better on the nearly tree-structured \emph{maze}.

\begin{table}[t]
    \centering
    \caption{Ablation and performance comparison results.}
    \label{tab:results}
    \setlength{\tabcolsep}{3.5pt}
    \begin{tabular}{@{}crcccccc@{}}
        \toprule
        & \multicolumn{1}{c}{\multirow{2}{*}{\textbf{Solver}}}
        & \multicolumn{1}{c}{\textbf{Cost}}
        & \multicolumn{1}{c}{\multirow{2}{*}{\textbf{Gap}$\downarrow$}}
        & \multicolumn{1}{c}{\multirow{2}{*}{\textbf{Feasible}$\uparrow$}}
        & \multicolumn{1}{c}{\multirow{2}{*}{\textbf{NPI}$\downarrow$}}
        & \multicolumn{1}{c}{\textbf{Closed}}
        & \multicolumn{1}{c}{\textbf{Eval.}} \\
        & & \multicolumn{1}{c}{\textbf{Regret}$\downarrow$}
        & & & & \multicolumn{1}{c}{\textbf{Nodes}}
        & \multicolumn{1}{c}{\textbf{Walks}} \\
        \midrule
        \multirow{6}{*}{\rotatebox{90}{\textit{rand}}}
        & \textbf{GHOST\cite{tang2026ghost}} & 2.1\% & 50.4\% & 43/60 & 0.386 & -- & 294 \\
        & \textbf{MICP\cite{marcucci2024graphs}} & 24.5\% & 100.0\% & 38/60 & 0.603 & -- & -- \\
        & \textbf{Ours (BF)} & \textbf{0.0\%} & \textbf{18.3\%} & \textbf{60/60} & \textbf{0.011} & 8.6k & 666 \\
        & \textbf{Ours (DF)} & 1.8\% & 21.7\% & \textbf{60/60} & 0.023 & 1.2k & 429 \\
        \cmidrule(lr){2-8}
        & \textbf{Alt-G} & 28.0\% & 33.0\% & \textbf{60/60} & 0.205 & 736 & 1 \\
        & \textbf{Alt-H} & 26.8\% & 54.7\% & \textbf{60/60} & 0.197 & 749k & 56 \\
        \midrule
        \multirow{6}{*}{\rotatebox{90}{\textit{maze}}}
        & \textbf{GHOST\cite{tang2026ghost}} & \textbf{0.0\%} & 3.4\% & 12/60 & 0.802 & -- & 23 \\
        & \textbf{MICP\cite{marcucci2024graphs}} & 2.2\% & 100.0\% & \textbf{60/60} & 0.163 & -- & -- \\
        & \textbf{Ours (BF)} & 0.7\% & 3.6\% & \textbf{60/60} & 0.014 & 53.5k & 168 \\
        & \textbf{Ours (DF)} & \textbf{0.0\%} & \textbf{3.1\%} & \textbf{60/60} & \textbf{0.005} & 1.3k & 173 \\
        \cmidrule(lr){2-8}
        & \textbf{Alt-G} & 2.5\% & 5.2\% & \textbf{60/60} & 0.027 & 453 & 1 \\
        & \textbf{Alt-H} & 2.5\% & 52.7\% & \textbf{60/60} & 0.027 & 871k & 1 \\
        \midrule
        \multirow{6}{*}{\rotatebox{90}{\textit{iiwa}}}
        & \textbf{GHOST\cite{tang2026ghost}} & 3.9\% & 64.6\% & 44/60 & 0.383 & -- & 37 \\
        & \textbf{MICP\cite{marcucci2024graphs}} & -- & -- & 0/60 & 1.000 & -- & -- \\
        & \textbf{Ours (BF)} & \textbf{0.2\%} & \textbf{62.4\%} & \textbf{60/60} & \textbf{0.042} & 830 & 54 \\
        & \textbf{Ours (DF)} & 4.0\% & 64.4\% & \textbf{60/60} & 0.062 & 98 & 30 \\
        \cmidrule(lr){2-8}
        & \textbf{Alt-G} & 10.6\% & 64.0\% & \textbf{60/60} & 0.111 & 43 & 1 \\
        & \textbf{Alt-H} & 6.2\% & 65.1\% & \textbf{60/60} & 0.083 & 314k & 37 \\
        \bottomrule
    \end{tabular}
\end{table}

\section{Mobile-Manipulator Inspection Case Study}\label{sec:case-study}
We demonstrate Ours (BF) using a mobile manipulator comprising the same 7-DoF arm as in the \textit{iiwa} domain, a wrist-mounted camera, and a 3-DoF holonomic planar base with pose $(x_b,y_b,\theta_b)\in\mathbb{R}^2\times S^1$.
As shown in Fig.~\ref{fig:demo}, the robot departs from a charging dock, navigates the workspace to acquire eight inspection images, and returns to the dock.
We precompute the GCS $G=(V,E,\mathcal{X})$ as coupled navigation and inspection subgraphs, totaling 92 vertices and 208 directed edges.
The navigation subgraph contains collision-free regions in the 3-DoF base configuration space, lifted to the full 10-DoF robot configuration space by fixing the manipulator in a retracted posture.
The inspection subgraph contains convex regions generated by IRIS-NP~\cite{petersen2023growing} in the 7-DoF manipulator configuration space, with the base fixed at one of the inspection base poses represented in the navigation subgraph.
The charging-dock vertex serves as the root $r\in V$, and each inspection task $i=1,\ldots,8$ is represented by a nonempty service-vertex subset $V_i\subseteq V\setminus\{r\}$, where each $v\in V_i$ encodes one feasible combination of base pose and manipulator region for acquiring image $i$.
The GCS construction took $0.4\,\mathrm{s}$ and its LBG precomputation took $6.0\,\mathrm{s}$.

\begin{figure}
    \centering
    \includegraphics[width=\linewidth]{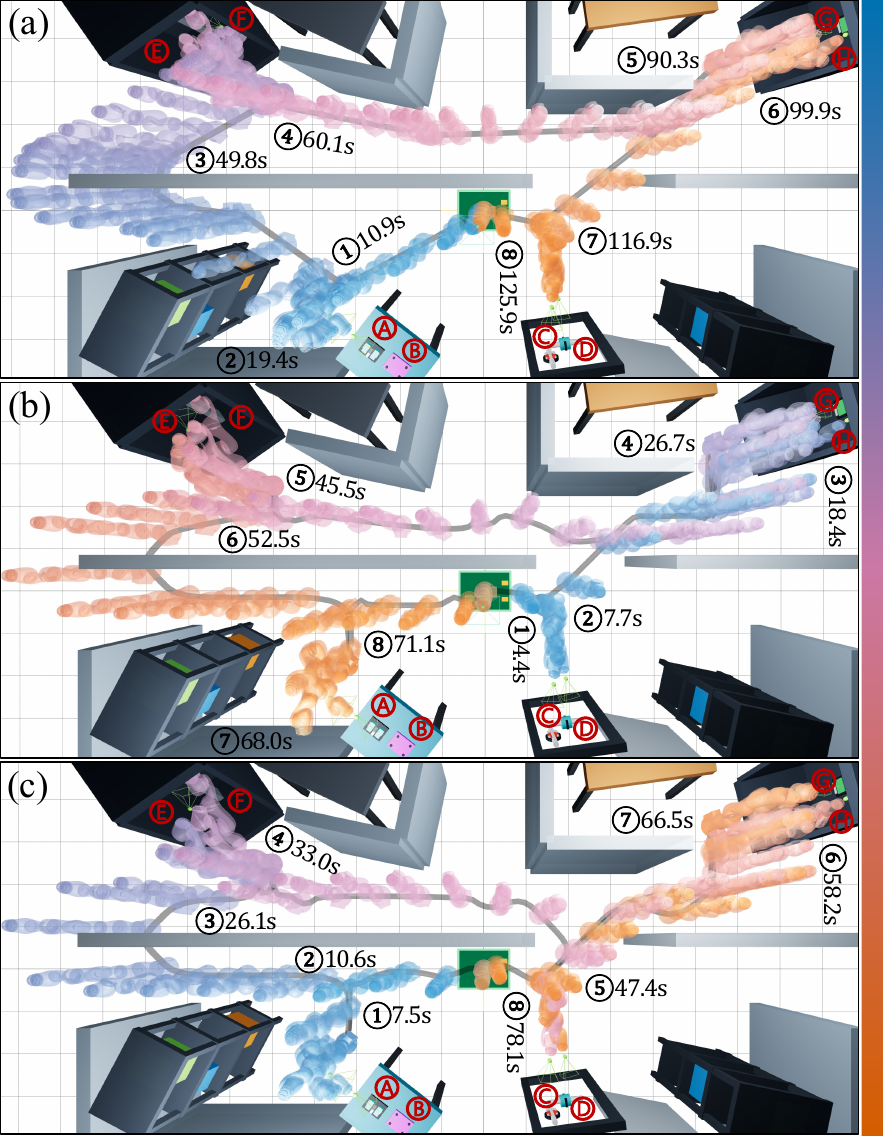}
    \caption{Mobile-manipulator inspection: (a) PRM-based generalized-TSP baseline, (b) Ours (BF), and (c) Ours (BF) with LTL$_f$~\cite{degiacomo2013linear} action precedences. Translucent arm poses encode normalized active-motion progress from blue (early) to orange (late), as shown by the right color bar. Red circled A--H mark tasks; black circled 1--8 show execution order and simulation timestamps.}
    \label{fig:demo}
\end{figure}

We formulate this inspection task as a Steiner-TSP on the precomputed GCS $G$ by generalizing the coverage constraint in Eqn.~\eqref{eq:steiner-tsp-on-gcs:cover} so that the rooted closed walk must visit at least one vertex in each $V_i$.
Across the eight tasks, the alternatives encoded by $V_1,\ldots,V_8$ yield 96 sensing-mode assignments; if every $V_i$ is a singleton, the formulation reduces to Eqn.~\eqref{eq:steiner-tsp-on-gcs} with $V_t=\{r\}\cup\bigcup_iV_i$.
The search tracks which tasks have been covered and represents each uncovered $V_i$ in the connected-flow cost-to-go relaxation by the union $\bigcup_{v\in V_i}E_v$; all other machinery is unchanged, and the search jointly selects the modes, order, walk, and trajectory.
We adopt the same benchmark trajectory model in Sec.~\ref{subsec:instances}.

Ours (BF) finds an initial $91.4\,\mathrm{s}$ greedy solution in $0.35\,\mathrm{s}$ with a $19.03\%$ gap, and improves it to $78.4\,\mathrm{s}$ in $1.72\,\mathrm{s}$ with a $5.64\%$ gap after 642 closed search nodes (Fig.~\ref{fig:demo}(b)).
For comparison, we construct a PRM~\cite{kavraki1996probabilistic} for the 3-DoF base (in $2.0\,\mathrm{s}$) with the manipulator fixed in a retracted posture, and augment it with sampled feasible 7-DoF manipulator configurations at the same inspection base poses in the GCS using OMPL~\cite{sucan2012open}.
We then solve the resulting generalized-TSP~\cite{saha2006planning} for the same eight inspection tasks.
It returns a feasible $133.1\,\mathrm{s}$ solution in $31.6\,\mathrm{s}$ without an optimality guarantee (Fig.~\ref{fig:demo}(a)).

Fig.~\ref{fig:demo}(c) further demonstrates the versatility of Ours (BF) by augmenting the inspection task with a linear temporal logic over finite traces (LTL$_f$) specification~\cite{degiacomo2013linear}.
It additionally enforces $\{\mathrm{A},\mathrm{B}\}\prec\{\mathrm{E},\mathrm{F}\}\prec\mathrm{D}\prec\{\mathrm{G},\mathrm{H}\}\prec\mathrm{C}$ for the task, where $\prec$ denotes a precedence relation.
Following prior works over product GCSs~\cite{kurtz2023temporal,wei2025hierarchical}, we compile the specification into a deterministic finite automaton and explicitly track its state in each search node. 
Ours (BF) finds an initial $87.4\,\mathrm{s}$ greedy incumbent in $0.4\,\mathrm{s}$ and, after closing 17 search nodes, improves it to $82.3\,\mathrm{s}$ in about $1.3\,\mathrm{s}$ with a $10.12\%$ gap.

%% file: conclusion.tex
\section{Conclusion}
We formalized Steiner-TSP on GCS and proposed a unified branch-and-bound search over rooted walk prefixes.
With uniformly positive GCS vertex costs, best-first traversal terminates after finitely many expansions on every feasible instance without an initial incumbent, whereas depth-first traversal does so once a finite incumbent is available.
A global lower bound provides an $\epsilon$-optimality certificate for either strategy.
Future work will apply the proposed search to physical robots for energy-aware service and maintenance, study online replanning via incremental space-time GCS~\cite{tang2026search} updates with ultrafast convex-set construction~\cite{werner2025superfast}, and develop scalable multi-robot planning through coordinated robot-local product GCSs.